\documentclass{article}

\PassOptionsToPackage{numbers, compress}{natbib}
\usepackage[dvipsnames]{xcolor}

\definecolor{myblue}{RGB}{31, 119, 180}
\definecolor{myorange}{RGB}{255, 127, 14}
\definecolor{mygreen}{RGB}{44, 160, 44}
\definecolor{myred}{RGB}{214, 39, 40}
\definecolor{myyellow}{RGB}{230,194,0}

\usepackage[final]{neurips_2024}

\usepackage[utf8]{inputenc} %
\usepackage[T1]{fontenc}    %
\usepackage[allcolors=NavyBlue,colorlinks=true,backref=page]{hyperref}       %
\usepackage{url}            %
\usepackage{booktabs}       %
\usepackage{amsfonts}       %
\usepackage{nicefrac}       %
\usepackage{microtype}      %
\usepackage{xspace}
\usepackage{subcaption}

\usepackage{wrapfig}

\usepackage[noend]{algpseudocode}
\usepackage{algorithm}

\usepackage{cite}
\usepackage{comment}
\usepackage{prettyref}
\usepackage{amsfonts}
\usepackage{makecell}
\usepackage{adjustbox}
\usepackage{multicol}

\usepackage{amssymb}
\usepackage{mathtools}
\usepackage{amsthm}
\usepackage{amsmath}

\usepackage[textsize=footnotesize]{todonotes}
\usepackage{xspace}

\definecolor{colorcomment}{RGB}{160, 190, 210}%
\makeatletter
\algnewcommand{\LineComment}[1]{\Statex \hskip\ALG@thistlm \(\triangleright\) 
{\color{colorcomment}#1}}
\makeatother

\makeatletter
\algnewcommand{\IndentLineComment}[1]{\Statex \hskip\ALG@tlm \(\triangleright\) {\color{colorcomment}#1}}
\makeatother

\newcommand\policy{\ensuremath{\pi}}

\newcommand\state{s}

\newcommand\stateDist{d}

\newcommand\transDynamics{\mathcal{P}}

\newcommand\RFunc{R}

\newcommand\corpus{\mathcal{C}}

\newcommand\gradient{{g}}

\newcommand\VFunc{\ensuremath{\ensuremath{V}}}

\newcommand\QFunc{\ensuremath{\ensuremath{Q}}}

\newcommand\AFunc{\mathbi{A}}

\def\mathbi#1{\textbf{\em #1}}

\newcommand\stateSpace{\ensuremath{\mathcal{S}}}
\newcommand\action{\ensuremath{a}}

\newcommand\actionSpace{\ensuremath{\mathcal{A}}}

\newcommand\horizon{\ensuremath{T}}
\newcommand\ENum{\ensuremath{N}} %

\newcommand\critic{\ensuremath{{C}}\xspace}

\newcommand\MDP{\ensuremath{\mathcal{M}}}

\newcommand{\expct}[1]{\mathbb{E}\left[#1\right]}
\newcommand{\expctover}[2]{\mathbb{E}_{#1}\!\left[#2\right]}
\newcommand\RationaleBuffer{\ensuremath{\mathcal{B}}}

\newcommand{\TOPK}{\textsc{Top-K}\xspace}
\newcommand{\TOPP}{\textsc{Top-P}\xspace}
\newcommand{\TOPPK}{\textsc{Top-PK}\xspace}
\newcommand{\hypothesis}{\mathcal{Y}}

\newcommand{\E}{\mathbb{E}}
\newcommand{\given}{{\,|\,}}
\newcommand{\AP}{\mathtt{AP}}

\def \argmax {\mathop{\rm arg\,max}}
\def \argmin {\mathop{\rm arg\,min}}

\newcommand\Gen{\ensuremath{\policy_{{\theta}}}}

\newcommand\Vocabulary{\ensuremath{\mathcal{V}}}

\newcommand\CNum{\ensuremath{\mathcal{K}}} %

\newcommand{\algname}{\textsc{SPO}\xspace}
\newcommand{\bestN}{\textsc{BON}\xspace}

\newcommand{\fwname}{\textsc{DrugImprover}\xspace} %

\newif\iffinal
\finaltrue

\iffinal
    \newcommand{\fix}[1]{#1}
    \newcommand{\note}[1]{}
    \newcommand{\pref}[1]{}
    \newcommand{\XL}[1]{}
    \newcommand{\YC}[1]{}
    \newcommand{\XLinline}[1]{}
    \newcommand{\YCinline}[1]{}
    \newcommand{\SJ}[1]{}
    
\else
    \newcommand{\fix}[1]{{\color{red} #1}}
    
    \newcommand{\YC}[1]{\todo[fancyline,color=NavyBlue!40]{YC: #1}\xspace}
    \newcommand{\YCinline}[1]{\textcolor{NavyBlue}{[YC: #1]}\xspace}
    \newcommand{\XL}[1]{\todo[fancyline,color=Maroon!40]{XL: #1}\xspace}
    \newcommand{\XLinline}[1]{\textcolor{Maroon}{[XL: #1]}}
    \newcommand{\SJ}[1]{\todo[fancyline,color=Blue!40]{SJ: #1}\xspace}
    
    \newcommand{\note}[1]{{\color{purple}[XL: #1]}}
    
    \newcommand{\pref}[1]{{\color{blue}(\ref{#1})}}
\fi

\newcommand{\tabref}[1]{Table~\ref{#1}}
\newcommand{\figref}[1]{Fig.~\ref{#1}}

\newcommand{\secref}[1]{\S\ref{#1}}
\newcommand{\appref}[1]{Appendix~\ref{#1}}

\newcommand{\lemref}[1]{Lemma~\ref{#1}}

\newcommand{\paren} [1] {\ensuremath{ \left( {#1} \right) }}

\newcommand{\bracket}[1]{\left[#1\right]}
\newcommand{\tuple}[1]{\ensuremath{\left\langle #1 \right\rangle}}
\newcommand{\curlybracket}[1]{\ensuremath{\left\{#1\right\}}}

\theoremstyle{plain}
\newtheorem{theorem}{Theorem}[section]

\newtheorem{lemma}[theorem]{Lemma}

\theoremstyle{definition}
\newtheorem{definition}[theorem]{Definition}

\theoremstyle{remark}

\newcommand{\BON}{{\mathtt{BON}}}

\title{DrugImproverLLM: A Large Language Model for Drug Optimization with Fine-Tuning via Advantage-Alignment Policy Optimization}

\author{%
}

\title{
DrugImproverGPT: A Large Language Model for Drug Optimization with Fine-Tuning via Structured Policy Optimization
}

\author{Xuefeng Liu\textsuperscript{1}\thanks{Correspondence to: Xuefeng Liu <\href{mailto:xuefeng@uchicago.edu}{xuefeng@uchicago.edu}>.} ,~\textbf{Songhao Jiang\textsuperscript{1}},~\textbf{Siyu Chen\textsuperscript{2}},~\textbf{Zhuoran Yang\textsuperscript{2}},~\textbf{Yuxin Chen\textsuperscript{1}}\\~\textbf{Ian  Foster\textsuperscript{1,3}},~\textbf{Rick  Stevens\textsuperscript{1,3}} \\
\textsuperscript{1}Department of Computer Science, University of Chicago\\
\textsuperscript{2}Department of Statistics and Data Science, Yale University \\
\textsuperscript{3}Argonne National Laboratory
}

\begin{document}

\maketitle

\begin{abstract}
Finetuning a Large Language Model (LLM) is crucial for generating results towards specific objectives.
This research delves into the realm of drug optimization and {introduce} a novel reinforcement learning algorithm to finetune a drug {optimization} {LLM-based generative} model, enhancing the original drug across 
target objectives, while {retains the beneficial chemical properties of the original drug.}
This work is comprised of two primary components: {
(1) \fwname: A framework tailored for improving robustness and efficiency in drug optimization. It includes a LLM designed for drug optimization and a novel Structured Policy Optimization (\algname) algorithm, which is theoretically grounded. This algorithm offers a unique perspective for fine-tuning the LLM-based generative model by aligning the improvement of the generated molecule with the input molecule under desired objectives.
} 
(2) A dataset of 1 million compounds, each with OEDOCK docking scores 
on 5 human proteins associated with cancer cells
and {24 binding sites from SARS-CoV-2 virus}.
We conduct a comprehensive evaluation of \algname 
and demonstrate its effectiveness in improving the original drug across target properties. 
Our code and dataset will be publicly available at: \url{https://github.com/xuefeng-cs/DrugImproverGPT}.

\end{abstract}

\section{Introduction}\label{sec:intro}

The cost of discovering a new drug through conventional approaches is estimated to range from hundreds of millions to billions of dollars~\citep{dickson2009cost}. 
This high cost is due to the lengthy and resource-intensive nature of the drug discovery and development process, which involves multiple stages, including target identification, lead compound identification, preclinical testing, and clinical trials. 
Despite significant efforts, the overall success rate in drug discovery is relatively low, with many drug candidates failing to progress beyond the early stages of development. 
Additionally, the time required to identify an effective drug can vary from several years to over a decade, depending on the complexity of the disease and the efficiency of the drug discovery process.
Such concerns are driving a growing trend towards drug repurposing~\citep{avram2023drugcentral}, which involves using FDA-approved drugs for different diseases instead of developing new drugs from the ground up. Yet despite some successes~\citep{pushpakom2019drug}, the effectiveness of drug repurposing has been limited since the drug is usually designed specifically for treating a particular disease. 
However, the emergence of rapidly evolving virus variants~\citep{hadj2022covid}, such as those associated with SARS-CoV-2~\citep{yuki2020covid}, as well as drug resistant cancer cells~\citep{MANS2017}, has sparked increased interest and  an urgent need to expedite the discovery of effective drugs. 

In this work, we 
propose
a reinforcement learning (RL)-based drug optimization {algorithm} to adapt existing drugs to fast-evolving virus variants and cancer cells, helping to address 
the aforementioned limitations of drug discovery and drug repurposing.  
RL has achieved superhuman performance in domains such as \fix{c}hess~\citep{lai2015giraffe}, video games~\citep{mnih2013playing}, and \fix{r}obotics~\citep{brunke2022safe}.
However, despite promising early results~\citep{born2021paccmannrl, guimaraes2017objective, jin2020multi, neil2018exploring, tan2022drlinker, zhang2023universal}, RL has yet to attain similar levels of performance for complex real-life problems like drug discovery. 

We identified four challenges that have thus far prevented RL from impacting drug design:
1) \textit{Search space complexity}\/: 
An RL algorithm for drug discovery needs to demonstrate both sample and computational efficiency, but the overwhelming complexity of the search space \citep{polya2012combinatorial} renders RL incapable of adequately exploring potential effective actions and states required for policy learning.
2) \textit{Sparse rewards}\/:
In contrast to the continuous reward environment found in popular environments like DeepMind Control Suite~\citep{tassa2018deepmind} or Meta-World~\citep{yu2020meta}, drug generation operates within a sparse reward environment where rewards are only obtainable upon {a complete molecule}.
3) \textit{Complex scoring criteria}\/: 
Generated molecules must fulfill multiple criteria, including solubility and synthesizability, 
while also achieving a high docking score when targeting a specific site.
4) {
\textit{Preservation of original beneficial properties}\/:
Lastly, as drugs with similar chemical structures should exhibit similar biological/chemical effects~\citep{bender2004molecular}, it is crucial to strike a balance between optimizing the drug and preserving the original drug's beneficial properties.
}

\textbf{Our contributions.} {We present \fwname, a {LLM-based} drug optimization framework designed to improve various properties of an original drug in a robust and efficient manner. Within this workflow, we introduce the \textbf{S}tructured \textbf{P}olicy \textbf{O}ptimization (\textbf{\algname}) algorithm to utilize the advantage preference {of properties improvement} to perform direct policy {optimization}.
\fwname and \algname effectively tackle the challenges outlined above in the following manner:}
\fix{
(1.)  \emph{Designing an LLM for Drug Optimization.} In this study, we develop a large language model (LLM) tailored specifically for drug optimization, incorporating a specialized corpus and custom loss function, among other features.}
{(2.) \emph{Sample complexity, sparsity, and computational efficiency.}
Because of the sparse reward nature of the drug design,
pure RL often finds it challenging to learn a good policy due to the complexity of the search space.
To reduce this complexity, \algname employs an imitation-learning-based approach to {pre-train a LLM-based} generator policy {with desirable}
behavior based on prior experience {of} {designing}
drug SMILES~\citep{weininger1988smiles} {strings}.
\algname also addresses the problem of reward sparsity by {utilizing the \TOPK and \TOPP Beam Search from LLM}
to obtain estimated rewards for intermediate steps. Finally, because calculating the docking score through virtual screening ({such as} OEDOCK~\citep{kelley2015posit}) is computationally costly~\citep{clyde2023ai}, \fwname adopts a transformer-based surrogate model to obtain docking scores more efficiently.
(3.) \emph{Property preserving.} 
To preserve the original drug's beneficial properties,
throughout the optimization process, it is crucial to balance the preservation of the original drug's beneficial properties with the optimization of other chemical attributes. To achieve this, we use Tanimoto similarity as a critic to {maximize} the Tanimoto similarity between the original and generated drugs.
(4.) \emph{Finetuning.}
Our proposed \algname algorithm leverages the advantageous preference of a generated drug over the original drug based on multiple objectives as the policy gradient signal. It performs direct policy improvement on an LLM-based generative model and addresses the sparse reward problem through partial molecule improvement.
}

In summary, our contributions are:

$\bullet$ We introduce the \fwname framework, which includes a ground-up designed LLM tailored for drug optimization, along with a novel RL finetuning algorithm, \algname, designed for drug optimization with theoretical analysis.

$\bullet$ By conducting comprehensive experiments
 by comparing to competing baselines
 {with existing SOTA}
on real world 
{viral} and {cancer target}
proteins, we {demonstrate that \algname outperform existing SOTA baseline} 
while consistently enhances existing molecules/drugs across multiple desired objectives, leading to improved drug candidates.

$\bullet$ We release a drug optimization dataset comprising 1 million ligands along with their OEDOCK scores to five proteins associated with cancer:
colony stimulating factor 1 receptor (CSF1R) kinase domain (PDB ID: 6T2W),
NOP2/Sun RNA methyltransferase 2 (NSUN2) (AlphaFold derived),
RNA terminal phosphate cyclase B (RTCB) ligase (PDB ID: 7P3B),
and Tet methylcytosine dioxygenase 1 (TET1) (AlphaFold derived),
and Wolf-Hirschhorn syndrome candidate 1 (WHSC1) (PDB ID: 7MDN) and 24 high-affinity binding sites on protein SARS-CoV-2: 3CLPro~(PDBID: 7BQY) virus. \fix{See more details in \secref{app:dataset}.}

\section{Related Work}\label{sec:related}

\subsection{Imitation learning}

Imitation learning (IL) is a technique where an agent learns by mimicking an expert's actions. IL outperforms pure RL by reducing the complexity and sparsity of the search space~\citep{liu2023active}. Offline IL methods, like behavioral cloning~\citep{pomerleau1988alvinn}, require a dataset of expert trajectories but can lead to errors in the learner's policy. In contrast, interactive IL methods, such as DAgger~\citep{ross2011reduction} and AggreVaTe~\citep{ross2014reinforcement}, use Roll-in-Roll-out (RIRO) scheduling, where learners initially follow their policy but switch to expert guidance for trajectory completion. However, these methods assume constant expert availability, which is impractical, and do not allow returning to previous states once a rollout begins. Our work integrates RIRO with {\TOPK and \TOPP sampling}~\citep{liu2024erp}, creating a guide policy be same as learner policy that conducts roll-outs on any state and estimates returns, improving upon traditional RIRO limitations.

\subsection{Reinforcement learning} %

One prominent approach in drug design employs RL \citep{tan2022reinforcement} to maximize an expected reward defined as the sum of predicted property scores as generated by property predictors. 
In terms of representation, existing works in RL for drug design have predominantly operated on SMILES string representations \citep{born2021paccmannrl, guimaraes2017objective, neil2018exploring, olivecrona2017molecular, popova2018deep,  staahl2019deep,tan2022drlinker,wang2022reinforcement, zhang2023universal, zhou2019optimization} or graph-based representations \citep{atance2022novo,gottipati2020learning, jin2020multi,wu2022rlcg,you2018graph}.
\fix{Traditional methods, such as genetic algorithms modified for molecular graphs~\citep{yoshikawa2018population} and Monte Carlo tree search applied to molecular graphs~\citep{jensen2019graph}, have been employed. These studies primarily concentrate on the De Novo drug discovery challenge instead of drug optimization. }
In our research, we have chosen to employ the SMILES representation.
However, previous studies have primarily focused on discovering new drugs, frequently overlooking molecular structure constraints during policy improvement. This oversight can lead to drastic changes in structure or functional groups, making most of the generated compounds unsynthesizable. In contrast, our work concentrates on optimizing existing drugs while preserving their beneficial properties, rather than creating entirely new ones from scratch.
\fix{MIMOSA and DrugEx v3 ~\citep{fu2021mimosa, liu2023drugex} represents the most recent approaches based on graph structures for drug optimization. However, it falls short of finetuning capability for drug optimization, an issue that our work has successfully addressed.}

\subsection{RL finetuning}
{Finetuning the generator model is critical to achieve drug improvement.
Prior RL finetuning methodologies aimed at aligning models with feedback from both humans (RLHF~\citep{bai2022training,christiano2017deep, ibarz2018reward,touvron2023llama}) and AI (RLAIF~\citep{bai2022constitutional,leike2018scalable}), which have recently found applications in the fine-tuning of language models for tasks like text summarization~\citep{ bohm2019better, stiennon2020learning, wu2021recursively, ziegler2019fine}, dialogue generation~\citep{hancock2019learning, jaques2019way,  yi2019towards}, 
and language assistance~\citep{bai2022training}. 
A core feature of RLHF and RLAIF lies in training a reward model {or make direct policy improvement} from the comparison feedback, {such as Rank Responses to align Human Feedback~(RRHF)~\citep{yuan2023rrhf}, Reward Ranked Finetuning~(RAFT)~\citep{dong2023raft}, Preference Ranking Optimization~(PRO)~\citep{song2023preference}, and Direct Preference Optimization~(DPO)~\citep{rafailov2023direct}.}}
Differing from previous works, our approach does not rely solely on feedback from a single human or AI model; instead, we engage multiple critics to evaluate the advantage preference of the generated vs.\ original drug based on comprehensive assessments, including factors like solubility. 
Moreover, we make direct policy improvement by using the advantage preference {in standard RL} {instead of binary feedback}.

\subsection{{LLMs for drug optimization}}

Large language models (LLMs) have been employed in molecule generation~\citep{bagal2021molgpt, rothchild2021c5t5, frey2023neural} and drug discovery~\citep{bran2023transformers, liu2024erp}. In contrast, our work focuses on drug optimization, which requires maintaining the original drug's beneficial structure and properties rather than designing from scratch.
A notable work in the drug optimization domain is REINVEN\fix{T} 4~\citep{he2021molecular, he2022transformer, loeffler2024reinvent}, which has developed transformer-based generative models with a strong focus on pretraining. However pretraining facilitates the generation of molecules similar to those in the training dataset, it also inherently limits the scope of exploration due to biases present in the training data. {Furthermore, REINVENT 4 categorizes the features, resulting in insensitivity to numerical changes, and the molecules generated lack optimization for specific desired objectives, such as drug-likeness, among others.}
In contrast, {\fwname} employed LLMs as the generative model and refines the generation process further through the \algname algorithm to guarantee the improved property in the optimized drug compared to original drug.

\section{Preliminaries}

\paragraph{Markov decision process.} We consider a finite-horizon Markov Decision Process (MDP) $\MDP_0=\langle\stateSpace,\actionSpace,\transDynamics,\RFunc, \horizon\rangle$ with state space $\stateSpace$, action space $\actionSpace$, deterministic transition dynamics $\transDynamics: \stateSpace \times \actionSpace \rightarrow \stateSpace'$, unknown reward function $\RFunc: \stateSpace \times \actionSpace \rightarrow \bracket{0,1}$, and horizon $\horizon$. We assume access to a set of $K$ critics each represents a domain experts, defined as 
${\mathbf{\critic}}=\curlybracket{\critic^k}_{k=1}^\CNum$, {where}
$\critic: \state_{\horizon} \rightarrow \mathbb{R}$ {and} $\state_{\horizon}$ represents a final state. The policy $\policy:\stateSpace \rightarrow \actionSpace$ maps the current state to a distribution over actions. 
{Given an initial state distribution $\rho_0 \in \Delta(\stateSpace)$, we define $\stateDist_t^{\policy}$ as the distribution over states at time $t$ under policy $\pi$.}
The goal is to train a policy to maximize the expected long-term reward.  The quality of the policy can be measured by the $\QFunc$-value function $\QFunc^{\policy}:\stateSpace \times \actionSpace \rightarrow \mathbb{R}$ is defined as:
$\QFunc^{\policy}\paren{\state,\action}:= \mathbb{E}^{\policy}\bracket{\sum_{t=0}^{\horizon}\RFunc\paren{\state_t,\action_t}|\state_0=\state,\action_0=\action}$,
where the expectation is taken over the trajectory following $\policy$, and the value function is noted as:
$\VFunc^{\policy}\paren{\state}:= \expctover{\action\sim\policy\paren{\cdot|\state}}{\QFunc^{\policy}\paren{\state,\action}}$.

{

\textbf{LLM.}
Each training corpus includes a start token $\bracket{\text{BOS}}$, a sequence of tokens $\mathbf{y}$ where each $y_i \in \Vocabulary$, and a termination action $\bracket{\text{EOS}}$. Here, each action $\action \in \actionSpace$ is represented as a token $y$ in the Transformer's vocabulary $\mathcal{V}$, with $\mathcal{V} := \actionSpace$. Each molecule is represented by a sequence of tokens $\mathbf{y}$ to construct a SMILES~\citep{weininger1988smiles} string, and this applies to both partial and complete molecules.
Let  $\circ$ represents string concatenation, and let $\mathcal{V}^*$ denote the Kleene closure of $\mathcal{V}$.
We define the set of complete training corpus as: 
\begin{equation}
    \corpus := \curlybracket{\text{[BOS]} \circ \mathbf{v} \circ \text{[EOS]}~|~\mathbf{v}\in \mathcal{V}^*}.
\end{equation}
The LLM generator policy $\policy_{\theta}$, which is parameterized by a deep neural network (DNN) with learned weights $\theta$, is defined as a product of probability distributions:
$\policy_{\theta}\paren{\mathbf{y}|\mathbf{x}}=\prod_{t=1}^{|\mathbf{y}|} \policy_{\theta}\paren{y_t|\mathbf{x},\mathbf{y}_{<t}}$, where  $\policy_{\theta}\paren{y_t|\mathbf{x},\mathbf{y}_{<t}}=P\paren{y_{t}|\mathbf{y}_{<t},X}$ is a distribution of next token $y_t$, $\mathbf{y}_{<t} = \bracket{y_1, \cdots, y_{t-1}}$, and $\mathbf{x}$ represents an input sequence (prompt).
The decoding process in text generation is designed to identify the most probable hypothesis from all potential candidates by resolving the following optimization problem:
\begin{equation}
\mathbf{y}^{\star}=\argmax_{\mathbf{y}\in \hypothesis_{\horizon}} \log \policy_{\theta}\paren{\mathbf{y}|\mathbf{x}}. 
\end{equation}
{To estimate the expected reward for a partial molecule, we employ 
\TOPPK \citep{liu2024erp} to navigate the exponentially vast search space to form a complete valid molecule}.
For sampling the token $y_i\sim\TOPPK\paren{\mathbf{y}_{<i},p,k}|_{\mathbf{x}}$,
where
\TOPPK 
generates the sequence by recursively picking the top candidates at each step $i$ according to 
\begin{gather}\label{eq:toppk}
    \text{\TOPPK}\paren{\mathbf{y}_{<i}, p,k}|_{\mathbf{x}}= \actionSpace_{\mathbf{y}_{<i}},\\
    \text{~where~} \actionSpace_{\mathbf{y}_{<i}}=\curlybracket{y^1,\ldots,y^{j}} \in \mathcal{V}^j, \text{~and}\\ \notag
     j= \min \curlybracket{\argmin_{j'}\sum_{\fix{l}=1}^{j'}\policy_{\theta}\paren{y^{\fix{l}}|\mathbf{x},\mathbf{y}_{<i}} \geq p, \: k},  \notag
\end{gather}
where the candidates $y^1, \dots, y^{j'}, \dots, y^{|\mathcal{V}|} \in \mathcal{V}$ are indexed by descending order of $\pi_\theta(\cdot \given \mathbf{x},\mathbf{y}_{<i})$, $p\in (0,1]$ denotes the 
\fix{cumulative probability threshold} 
and $k$ represents the maximum number of candidates for the next tokens.
$\bestN$~\citep{gao2023scaling} sampling technique at inference time generates $N$ samples which are then ranked by the reward model. Then top ranked candidate is selected, which can expressed as 
\begin{align}
{\bestN}\paren{\mathbf{y}_{<i},N,\RFunc}|_{\textbf{x},p,k}= \max_{\mathbf{Y}_{j}\in \curlybracket{\mathbf{Y}_1,\cdots,\mathbf{Y}_{N}}} \RFunc\paren{ \mathbf{Y}_{j}},\\
\text{where } Y_j=\bracket{\mathbf{y}_{<i},y_i,\cdots,y_{\horizon}}_j,\\ \text{and } y_i \sim \TOPPK\paren{\mathbf{y}_{<i},p, k}|_{\mathbf{x}}.
\end{align}\label{eq:bestN}
}
\paragraph{Drug {optimization}.} 
We formalize the drug {optimization} problem within the framework of MDP. Given a dataset consisting of real-world structured sequences represented as SMILES~\citep{weininger1988smiles} strings, 
our objective is to train a {LLM-based} generative policy $\Gen$ to generate a high-quality sequence denoted as $\mathbf{y}_{\horizon}=\paren{y_1,\ldots,y_t,\ldots,y_{\horizon}},y_t\in \Vocabulary$, {and aim to outperforming an input sequence X in desired properties}.
The length of the output sequence, denoted as $\horizon$, represents the planning horizon. At time step $t$, the state $\state_{t-1}$ comprises the currently generated tokens $\paren{y_1,\ldots,y_{t-1}}$, and the action $\action$ corresponds to the next token $y_t$ to be selected. While the policy model $\Gen\paren{y_t|\mathbf{y}_{<t}, \fix{X}}$ operates in a stochastic manner, the state transition function $\transDynamics$ becomes deterministic once an action has been chosen. 
To estimate the $\QFunc$ value, we reference the REINFORCE algorithm \citep{williams1992simple}, which we define as 
$\QFunc\paren{\state= 
\mathbf{y}_{<\horizon},\action=y_{\horizon}}=\RFunc\paren{\mathbf{y}_{\horizon}}. $

\paragraph{Limitations of previous work.} 

1) Prior studies concentrated primarily on the discovery of new drugs from the ground up \citep{atance2022novo, popova2018deep, zhang2023universal}. In contrast, we focus on the relatively less explored, yet highly practical and significant, issue of drug optimization.
2) {There is currently no RL finetuning algorithm specifically designed for drug optimization problems.}
{3) The current state-of-the-art model, REINVENT~4, prioritizes pretraining with constrained similarity, thereby restricting its ability to explore molecular spaces with potentially high rewards beyond its training set.
Our drug optimization LLM, together with the Structured Policy Optimization approach, addresses these limitations.}

\section{The \fwname Framework} %
{In this work, we propose \fwname  as in \figref{fig:drugimprover_framework}, which}
comprises {two} major components: 
{(1) A large language model designed \fix{from the ground up} for drug optimization.}
{(2) A Structured Policy Optimization (\algname)
{algorithm with theoretical support.}
} 
{We introduce each part in details as follows.}

\subsection{{\fix{Designing \& }pretraining a LLM generator}}\label{sec:pretrain}
In drug optimization, firstly, we construct a molecule pair $\paren{X,Y}$,
where $X$ represents for original molecule, and $Y$ represents for the target optimized molecule. We randomly select non-duplicated pair $\paren{X,Y}$ from ZINC15~\citep{sterling2015zinc} dataset \fix{(See Appendix \ref{app:pretrain_data})}, and added the pair to the training set by meeting the following criteria of Tanimoto Similarity~\citep{bajusz2015tanimoto} and molecule scaffold~\citep{landrum2013rdkit}:
\begin{equation}
\text{Tanimoto}\paren{X,Y} > 0.5 \text{~~or~~}\text{Scaffold}(X) = \text{Scaffold}(Y), \notag
\end{equation}
The motivation for such a form is to provide an initialization for a diversified molecule pair while constraining the similarity between the pair. After obtaining the training set of molecule pairs, we formed the training corpus as follows:
\begin{equation}\label{eq:corpus}
\corpus=\curlybracket{
\tuple{S},\underbrace{x_1,\cdots,x_{\horizon}}_{\text{source molecule X}},\tuple{L},\underbrace{y_1,\cdots,y_{\horizon}}_{\text{target molecule Y}}
},
\end{equation}
where $\tuple{S}$ stands for the source ligand and $\tuple{L}$ stands for the target ligand. \fix{We include visualization towards the corpus in Appendix \ref{corpus_viz}.}
We enhance the training by concentrating on pairs of molecules through Causal Language Modeling (CLM)~\citep{vaswani2017attention}.
The parameters $\theta$ of the LLM generator $\policy_{\theta}$ are trained through the minimization of the negative log-likelihood (NLL) for the complete molecular pair across the entire training corpus.
This process is described as follows:
\begin{align}\label{eq:pretrain-loss}
    \text{NLL}\paren{X,Y} =  - \log P\paren{Y|X} = - \log \prod_{l=1}^{\horizon} P\paren{y_{l}|\mathbf{y}_{<l},X} 
    =-\sum_{l=1}^{\horizon} \log P\paren{y_{l}|\mathbf{y}_{<l},X},
\end{align}
where $\horizon$ signifies the total number of tokens related to $Y$. The NLL measures the likelihood of transforming a specific original molecule into a designated target molecule. 
Given the goal of drug optimization, we aim for the generated drugs to resemble the originals. Therefore, we have incorporated a regularization term into the loss in Equation \eqref{eq:pretrain-loss}, which penalizes the NLL if the sequence does not adhere to a specified similarity metric.
Finally, we propose the following loss function:
\begin{align}\label{eq:loss}
    \mathcal{L}=
    \frac{1}{|\corpus|}\sum_{\paren{X,Y}\in \corpus}
    (\lambda\cdot\text{NLL}\paren{X,Y} \fix{/} \paren{ \paren{1-\lambda}\cdot \text{Similarity}\paren{ X, Y})}, \lambda\in \fix{\paren{0,1}}.
\end{align}
Consequently, 
training the model with the loss function described in Eqn. \eqref{eq:loss} can generate the corresponding target molecule when provided with a source molecule.
However, this approach primarily focuses on maximizing likelihood without considering specific metrics of interest, making it unsuitable for optimizing objectives that differ from those in its training set, as encoded in $\policy_{\theta}$. 
Therefore, these generation algorithms cannot be directly applied to design molecules that fulfill various objectives, such as attaining a high docking score at a specific target site or improving upon specific desired metrics.
We aim to further refine the LLM model to generate specific improved outcomes using reinforcement learning techniques in the next phase.

\fix{

}

\subsection{Structured policy {optimization}\label{sec:4.2}
}
\begin{algorithm}[t]
    \caption{Structured Policy Optimization (\algname)
    }\label{alg:lops}
    \begin{algorithmic}[1] 
    \Require { LLM-based generator $\Gen$; roll-out policy $\policy_\beta$; a pre-train dataset $\RationaleBuffer$, {critics $\mathbf{\critic}$}.}
    \State Pre-train $\Gen$ using loss function~\eqref{eq:loss} and training corpus  \eqref{eq:corpus} through CLM objective.
    \State $\beta\leftarrow\theta$.

    \For{$n=1,\ldots,\ENum$}
            \State $\fix{X} \sim \rho_0$, { \text{where }$\rho_0\in \Delta\paren{\RationaleBuffer}$}.
            \State Generate $Y_{1:\horizon}=\paren{y_t,\ldots,y_{\horizon}}\sim \Gen\paren{\cdot|X}$.
        \LineComment{/* incorporating partial reward */}
            \State Compute advantage preference
            {$\RFunc^\text{AP}$}
            by incorporating partial molecule component. %
            \State Update generator $\theta$
            via policy gradient by \eqref{eq:adv:preference}\eqref{eq:gradient_update}.
        \State $\beta\leftarrow\theta$.
    \EndFor
    \end{algorithmic}
\end{algorithm}

\paragraph{{Normalized reward}.} 
\fix{In this work, we adopt the approach of \citet{liu2024erp} to construct the reward for multiple critics.}
Given an ensemble of critics 
{
\begin{align*}
\mathbf{C}{\paren{
\mathbf{y}_{\horizon}
}}=[
\critic^{\text{Druglikeness}}{\paren{\mathbf{y}_{\horizon}}},
\critic^{\text{Solubility}}{\paren{\mathbf{y}_{\horizon}}},
\critic^{\text{Synthesizability}}{\paren{\mathbf{y}_{\horizon}}}, {\critic^{\text{Docking}}}{\paren{\mathbf{y}_{\horizon}}}
],
\end{align*}
}
where  $\mathbf{y}_{\horizon}:=\state_{\horizon}, \critic: \fix{\mathbf{y}_{\horizon}} \rightarrow \mathbb{R}$. 
\fix{{We leverage the RDKit  \citep{landrum2016rdkit} chemoinformatics package to calculate the listed critics:}
\textbf{Druglikeness:} The druglikeness measure the likelihood of a molecule being suitable candidate for a drug.
\textbf{Solubility:} This metric assesses the likelihood of a molecule's ability to mix with water, commonly referred to as the water-octanol partition coefficient (LogP). 
\textbf{Synthetizability:} This parameter quantifies the ease (score of 1) or difficulty (score of 10) associated with synthesizing a given molecule \citep{ertl2009estimation}.
\textbf{Docking Score:} The docking score assesses the drug's potential to bind and inhibit the target site. 
To enable efficient computation, we employ a docking surrogate model (See \appref{app:surrogate_model}) to output this score.}
Here we design the reward function to align the drug optimization with multiple objectives.
For a fully generated SMILES sequence, we derive the following
{normalized}
reward function based on assessments from multiple critics with equal weight \citep{liu2024erp} as follows:
{
\begin{align}\label{eq:reward}
{\RFunc_c\paren{\mathbf{y}_{\horizon}}}:=&
{\RFunc_c\paren{
\mathbf{y}_{\horizon}|X}} =
\beta\cdot\text{Norm}\paren{\critic^{\text{Tanimoto}}\paren{X,
\mathbf{y}_{\horizon}}}\notag \\ 
&+
\sum_{i=0}^{|\mathbf{C}|-1}\lambda\cdot\text{Norm}\paren{\critic_i{\paren{\mathbf{y}_{\horizon}}}},
\end{align}
}
where $\lambda=\frac{1-\beta}{|\mathbf{\critic}|+1}$.
{We use Norm\footnote{{Here, we define Norm as min-max normalization to scale the attributes onto the range [-10, 10].}} to normalize different attributes
{onto the} same scale.}
In this study, we employ the Tanimoto similarity calculation $\critic^{\text{Tanimoto}}$ to quantify the chemical similarity
between the generated compound and the original drug. 
Essentially, this calculation involves first computing Morgan Fingerprints~\citep{rogers2010extended} for each molecule and then measuring the 
{Jaccard distance~\citep{jaccard1912distribution} (i.e., intersection over union)}
 between the two fingerprints. \YC{expand the definition if still having space}

\begin{figure*}[t]
    \begin{subfigure}{1\textwidth}
        \centering
        \includegraphics[
        height=4.8cm, 
        clip={0,0,0,0}]{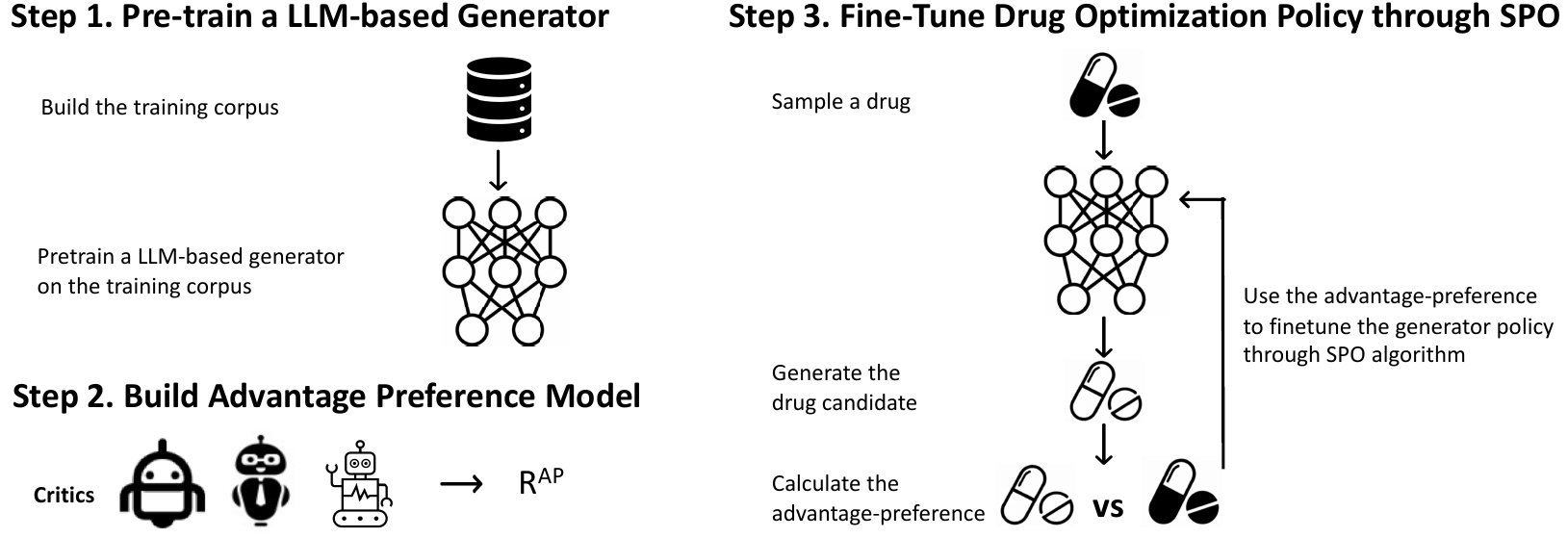}
    \end{subfigure}
    \caption{{\fwname framework.} 
    {It comprises two major components: {(1)  A large language model designed for drug optimization.} (2) A Structured Policy Optimization (\algname) algorithm aims to fine-tune the LLM-based generator for drug improvement across desired properties.\YCinline{Fine-tune} 
    }
    }
  \label{fig:drugimprover_framework}
\end{figure*}

\paragraph{Structured policy gradient \fix{with \emph{partial} molecule improvement}.} 
The return, denoted as $\QFunc^{\policy}$, often exhibits significant variance across multiple episodes. One approach to mitigate this issue is to subtract a baseline $b\paren{\state}$ from each $\QFunc$. The baseline function can be any function, provided that it remains invariant with respect to $\action$. For a generator policy $\Gen$, the advantage function \citep{sutton1999policy} is defined as follows:
    $\AFunc^{\Gen}\paren{\state,\action} = \QFunc^{\Gen}\paren{\state,\action} - b\paren{\state}$.
A natural choice for the baseline is the value function $\VFunc^{\policy}\paren{\state}$, which represents the expected reward at a given state $\state$ under policy $\policy$. The value function can be expressed as follows:
\begin{align}
\VFunc\paren{\state}&=\expctover{\action\sim\Gen\paren{\state}}{\QFunc(\state,\action)}
=\expctover{y_t\sim\Gen\paren{\mathbf{y}_{<t}}}{\QFunc(
\mathbf{y}_{<t},y_t)}.\notag
\end{align}
Thus, we have advantage function as 
\begin{align}
     \AFunc^{\Gen}\paren{\state,\action} =\AFunc^{\Gen}\paren{\mathbf{y}_{<t},y_{t}}
     =\QFunc^{\Gen}\paren{\mathbf{y}_{<t},y_t}-\VFunc^{\Gen}\paren{\mathbf{y}_{<t}}.\notag
\end{align}

{
Here we employ the one-step RL~\citep{brandfonbrener2021offline, peng2019advantage} method and regard the drug optimization method as a {sequence to sequence} language generation task. 
Rather than treating each token as an individual action, we treat the entire sequence $Y_{1:\horizon}=\mathbf{y}_{\horizon}$ as a single action generated by the policy $\Gen$. Subsequently, we receive rewards from critics, and the episode concludes. This leads to the formulation of our advantage function as follows:}
\begin{align}\label{eq:advantage}
\AFunc^{\Gen}\paren{\state,\action}&=\QFunc^{\Gen}\paren{\state_0,
\mathbf{y}_{\horizon}}-\VFunc^{\Gen}\paren{\state_0}\\
&=\RFunc_c\paren{Y_{1:\horizon}}-\RFunc_c\paren{X},
\end{align}
where $\state_0=X$ is the initial molecule sequence drawn from the distribution $\rho_0$, which corresponds to our buffer known as $\RationaleBuffer$ {containing selected SMILES strings}.
Thus, the advantage preference of the generated versus the original drug is
\begin{equation}\label{eq:advantage_preference_temp}
r^\AP\paren{Y_{1:\horizon},X} ={\RFunc_c\paren{Y_{1:\horizon}} - \RFunc_c\paren{X}},
\end{equation}
Nonetheless, the reward function only supports a reward value for a completed sequence for global optimization. \fix{In contrast with previous approaches}, we also aim to make improvement on partial molecule for local optimization by uniformly sample a subsquence $Y_{1:j},j \sim \mathcal{U}([T])$. 
To achieve this, we employ a Roll-in-Roll-out (RIRO) \citep{ross2014reinforcement, cheng2020policy,liu2023active,liu2023blending} scheduling, utilizing a roll-out policy denoted as $\policy_\beta$ (same as learner policy in our experiment) to sample the unknown last $\horizon-j$ tokens through \fix{$\bestN$} approach~\eqref{eq:bestN}. 
\fix{We propose a novel notion of advantage function, termed Advantage Preference (AP), by incorporating partial molecule improvement into \eqref{eq:advantage_preference_temp}:}

\vspace{+0.2cm}
\begin{definition}[Advantage preference] {We define advantage preference as a combination of rewards from both complete and partial molecules:} 
\vspace{+0.2cm}
\begin{align}\label{eq:advantage_preference}
    R^{\AP}(Y_{1:T}, X) = \frac 1 2 \E_{j\in \mathcal{U}([T])} r_{\BON(j)}^{\AP}(Y_{1:T}, X) + \frac 1 2 r^\AP(Y_{1:T},X),
\end{align}
where 
\begin{align*}
r_{\BON(j)}^{\AP}(Y_{1:T}; X_{1:T})=&\E_{Y_{j+1:T}\sim \BON(Y_{1:j}),X_{j+1:T}\sim \BON(X_{1:j})}[ 
R_c(Y_{1:T}) - R_c(X_{1:T})\given X_{1:j}, Y_{1:j}]
\end{align*}
and $r^{\AP}(Y_{1:T}; X_{1:T})$ is defined in Eq. \eqref{eq:advantage_preference_temp}.
\end{definition}

The \text{advantage preference} of \eqref{eq:advantage_preference} will be employed directly in the policy gradient~\eqref{eq:adv:preference} to finetune the {generator} policy $\Gen$. The rationale behind the advantage preference is
to produce a sequence that surpasses the initial state sequence $\state_0$ in every objective.
{In this work, our objective is to maximize the expected final advantage preference compared to the original drug $\fix{X}$
at the end of the sequence as follows
\begin{equation}
 J\paren{\theta}=\expctover{{X\sim\rho_0, Y_{1:T}\sim \pi_\theta(\cdot\given X)}}{\RFunc^\AP(Y_{1:T},X)},
 \label{eq:J(theta)}
\end{equation}
}
Thus, we have gradient $g$ as follows:
\fix{
\begin{align}\label{eq:adv:preference}
     \textstyle \mathbb{E}_{X\sim {\rho_0},{Y_{1:\horizon}\sim \Gen\paren{\cdot|X}} } 
     \left[ \nabla_{\theta}\log \Gen\paren{Y_{1:\horizon}|X}
     \cdot \RFunc^\AP\paren{X,Y_{1:\horizon}}\right],
\end{align}
}
{where $Y_{1:\horizon}$ is the generated sequence 
from $\Gen$ and $X$ is the original drug}. As the expectation $\expct{\cdot}$ can be approximated through sampling techniques, we proceed to update the generator's parameters as follows:
\YC{Move definition of $g$ (from Lemma 5.2) here}

\begin{equation}\label{eq:gradient_update}
    \theta\leftarrow\theta+\alpha_n\gradient,
\end{equation}
where $\alpha\in \mathbb{R}^{+}$ denotes the learning rate at $n$-th episode.

\begin{table*}[t!]
\setlength{\tabcolsep}{4pt}
   \centering
    { %
    \scalebox{0.65}{
    \begin{tabular}{l l c c c c c c c c c c }
        \toprule
        \textbf{Target} %
        & \textbf{Algorithm}
        & {\makecell[c]{Avg \\ Norm Reward~$\uparrow$}}
        & {\makecell[c]{Avg Top 10 \% \\ Norm Reward~$\uparrow$}}
        & {\makecell[c]{Docking ~$\downarrow$}}
        & {\makecell[c]{Druglikeliness ~$\uparrow$}}
        & {\makecell[c]{Synthesizability ~$\downarrow$}}
        & {\makecell[c]{Solubility ~$\uparrow$}}
        \\
        \midrule
        \makecell[l]{\textbf{3CLPro}} %
        &  \textbf{\makecell[l]{Original}}
        &  \makecell[l]{0.524}
        &  \makecell[l]{{0.689}}
        &  \makecell[l]{\underline{-8.687}}
        &  \makecell[l]{0.654}
        &  \makecell[l]{3.097 }
        &  \makecell[l]{2.455}
        \\
        (PDBID:
        &  \textbf{\makecell[l]{MMP \citep{loeffler2024reinvent}}}
        &  \makecell[l]{0.564 $\pm$ 0.003}
        &  \makecell[l]{0.680 $\pm$ 0.001}
        &  \makecell[l]{-8.184 $\pm$ 0.069}
        &  \makecell[l]{0.672 $\pm$ 0.003}
        &  \makecell[l]{2.658 $\pm$ 0.006}
        &  \makecell[l]{3.114 $\pm$ 0.067}
        \\
       \ 7BQY)
        &  \textbf{\makecell[l]{Similarity ($\geq$ 0.5) \citep{loeffler2024reinvent}}}
        &  \makecell[l]{0.572 $\pm$ 0.001}
        &  \makecell[l]{{0.686} $\pm$ 0.001}
        &  \makecell[l]{-8.158 $\pm$ 0.007}
        &  \makecell[l]{\textbf{0.686} $\pm$ 0.004}
        &  \makecell[l]{\underline{2.583} $\pm$ 0.010}
        &  \makecell[l]{3.121 $\pm$ 0.018}
        \\
        \textbf{ }
        &  \textbf{\makecell[l]{Similarity ([0.5, 0.7)]) \citep{loeffler2024reinvent}}}
        &  \makecell[l]{0.575 $\pm$ 0.002}
        &  \makecell[l]{0.686 $\pm$ 0.004}
        &  \makecell[l]{-8.171 $\pm$ 0.053}
        &  \makecell[l]{0.676 $\pm$ 0.002}
        &  \makecell[l]{2.588 $\pm$ 0.018}
        &  \makecell[l]{3.309 $\pm$ 0.032}
        \\
        \textbf{ }
        &  \textbf{\makecell[l]{Similarity ($\geq$ 0.7) \citep{loeffler2024reinvent}}}
        &  \makecell[l]{0.560 $\pm$ 0.002}
        &  \makecell[l]{0.677 $\pm$ 0.002}
        &  \makecell[l]{-8.187 $\pm$ 0.024}
        &  \makecell[l]{ 0.668 $\pm$ 0.007}
        &  \makecell[l]{2.699 $\pm$ 0.007}
        &  \makecell[l]{3.120 $\pm$ 0.013}
        \\
        \textbf{ }
        &  \textbf{\makecell[l]{Scaffold \citep{loeffler2024reinvent}}}
        &  \makecell[l]{0.552 $\pm$ 0.004}
        &  \makecell[l]{0.678 $\pm$ 0.009}
        &  \makecell[l]{-8.081 $\pm$ 0.049}
        &  \makecell[l]{0.675 $\pm$ 0.002}
        &  \makecell[l]{2.741 $\pm$ 0.014}
        &  \makecell[l]{3.002 $\pm$ 0.040}
        \\
        \textbf{ }
        &  \textbf{\makecell[l]{Scaffold Generic \citep{loeffler2024reinvent}}}
        &  \makecell[l]{\underline{0.567} $\pm$ 0.001}
        &  \makecell[l]{{0.680} $\pm$ 0.007}
        &  \makecell[l]{-8.078 $\pm$ 0.056}
        &  \makecell[l]{\underline{0.680} $\pm$ 0.005}
        &  \makecell[l]{2.613 $\pm$ 0.002}
        &  \makecell[l]{3.173 $\pm$ 0.046}
        \\
        \textbf{ }
        &  \textbf{\makecell[l]{Molsearch \citep{sun2022molsearch}}}
        &  \makecell[l]{0.518  $\pm$ 0.002} 
        &  \makecell[l]{\textbf{0.693 }  $\pm$ 0.003}
        &  \makecell[l]{{-8.506}  $\pm$ 0.038}
        &  \makecell[l]{0.656  $\pm$ 0.004}
        &  \makecell[l]{3.110  $\pm$ 0.010}
        &  \makecell[l]{2.448 $\pm$ 0.032}
        \\
        \textbf{ }
        &  \textbf{\makecell[l]{MIMOSA \citep{fu2021mimosa}}}
        &  \makecell[l]{0.530  $\pm$ 0.003} 
        &  \makecell[l]{0.690  $\pm$ 0.005} 
        &  \makecell[l]{\textbf{-8.764 } $\pm$ 0.048}
        &  \makecell[l]{0.649  $\pm$ 0.003}
        &  \makecell[l]{3.148  $\pm$ 0.023}
        &  \makecell[l]{2.732  $\pm$ 0.027}
        \\
        \textbf{ }
        &  \textbf{\makecell[l]{DrugEx v3 \citep{liu2023drugex}}}
        &  \makecell[l]{0.532  $\pm$ 0.003} 
        &  \makecell[l]{0.653  $\pm$ 0.004} 
        &  \makecell[l]{{-8.089 } $\pm$ 0.039}
        &  \makecell[l]{0.583  $\pm$ 0.005}
        &  \makecell[l]{3.095  $\pm$ 0.018}
        &  \makecell[l]{\textbf{3.932}  $\pm$ 0.031}
        \\
        \textbf{ }
        &  \textbf{\makecell[l]{\fwname (Ours)}}
        &  \makecell[l]{\textbf{0.601} $\pm$ 0.003} 
        &  \makecell[l]{{\underline{0.692}} $\pm$ 0.003} 
        &  \makecell[l]{{-8.163}  $\pm$ 0.034}
        &  \makecell[l]{{0.676}  $\pm$ 0.004}
        &  \makecell[l]{\textbf{2.381} $\pm$ 0.011}
        &  \makecell[l]{\underline{3.673}  $\pm$ 0.024}
        \\
        \bottomrule
        \textbf{RTCB}
        &  \textbf{\makecell[l]{Original}}
        &  \makecell[l]{0.538}
        &  \makecell[l]{{0.705}}
        &  \makecell[l]{-8.538}
        &  \makecell[l]{0.716}
        &  \makecell[l]{2.984}
        &  \makecell[l]{2.283}
        \\
        (PDBID:
        &  \textbf{\makecell[l]{MMP \citep{loeffler2024reinvent}}}
        &  \makecell[l]{0.583 $\pm$ 0.000}
        &  \makecell[l]{0.700 $\pm$ 0.001}
        &  \makecell[l]{-8.466 $\pm$ 0.012}
        &  \makecell[l]{0.709 $\pm$ 0.001}
        &  \makecell[l]{2.599 $\pm$ 0.004}
        &  \makecell[l]{2.978 $\pm$ 0.021}
        \\
        \ 4DWQ)
        &  \textbf{\makecell[l]{Similarity ($\geq$ 0.5) \citep{loeffler2024reinvent}}}
        &  \makecell[l]{\underline{0.593} $\pm$ 0.004}
        &  \makecell[l]{0.705 $\pm$ 0.001}
        &  \makecell[l]{-8.581 $\pm$ 0.096}
        &  \makecell[l]{0.715 $\pm$ 0.007}
        &  \makecell[l]{2.561 $\pm$ 0.005}
        &  \makecell[l]{3.065 $\pm$ 0.048}
        \\
        \textbf{ }
        &  \textbf{\makecell[l]{Similarity ([0.5, 0.7)]) \citep{loeffler2024reinvent}}}
        &  \makecell[l]{0.591 $\pm$ 0.002}
        &  \makecell[l]{0.709 $\pm$ 0.003}
        &  \makecell[l]{-8.526 $\pm$ 0.002}
        &  \makecell[l]{0.710 $\pm$ 0.006}
        &  \makecell[l]{2.561 $\pm$ 0.001}
        &  \makecell[l]{3.116 $\pm$ 0.089}
        \\
        \textbf{ }
        &  \textbf{\makecell[l]{Similarity ($\geq$ 0.7) \citep{loeffler2024reinvent}}}
        &  \makecell[l]{0.585 $\pm$ 0.001}
        &  \makecell[l]{0.705 $\pm$ 0.004}
        &  \makecell[l]{-8.584 $\pm$ 0.019}
        &  \makecell[l]{0.723 $\pm$ 0.000}
        &  \makecell[l]{2.604 $\pm$ 0.003}
        &  \makecell[l]{2.841 $\pm$ 0.009}
        \\
        \textbf{ }
        &  \textbf{\makecell[l]{Scaffold \citep{loeffler2024reinvent}}}
        &  \makecell[l]{0.581 $\pm$ 0.001}
        &  \makecell[l]{0.701 $\pm$ 0.004}
        &  \makecell[l]{-8.524 $\pm$ 0.011}
        &  \makecell[l]{0.718 $\pm$ 0.003}
        &  \makecell[l]{2.618 $\pm$ 0.021}
        &  \makecell[l]{2.840 $\pm$ 0.018}
        \\
        \textbf{ }
        &  \textbf{\makecell[l]{Scaffold Generic \citep{loeffler2024reinvent}}}
        &  \makecell[l]{0.592 $\pm$ 0.003}
        &  \makecell[l]{0.704 $\pm$ 0.004}
        &  \makecell[l]{-8.590 $\pm$ 0.033}
        &  \makecell[l]{0.725 $\pm$ 0.001}
        &  \makecell[l]{2.542 $\pm$ 0.018}
        &  \makecell[l]{2.916 $\pm$ 0.004}
        \\
        \textbf{ }
        &  \textbf{\makecell[l]{Molsearch \citep{sun2022molsearch}}}
        &  \makecell[l]{0.548 $\pm$ 0.002} 
        &  \makecell[l]{{0.731} $\pm$0.002}
        &  \makecell[l]{-8.750  $\pm$ 0.028}
        &  \makecell[l]{\underline{0.730} $\pm$ 0.003}
        &  \makecell[l]{2.981 $\pm$ 0.014}
        &  \makecell[l]{2.290 $\pm$ 0.027}
        \\
        \textbf{ }
        &  \textbf{\makecell[l]{MIMOSA \citep{fu2021mimosa}}}
        &  \makecell[l]{0.553 $\pm$ 0.002} 
        &  \makecell[l]{0.721  $\pm$ 0.002} 
        &  \makecell[l]{\underline{-8.980} $\pm$ 0.039}
        &  \makecell[l]{0.716 $\pm$ 0.002}
        &  \makecell[l]{3.066 $\pm$ 0.017}
        &  \makecell[l]{2.491 $\pm$ 0.019}
        \\
        \textbf{ }
        &  \textbf{\makecell[l]{{DrugEx v3} \citep{liu2023drugex}}}
        &  \makecell[l]{0.642 $\pm$ 0.002} 
        &  \makecell[l]{\underline{0.754}  $\pm$ 0.002} 
        &  \makecell[l]{{-8.762} $\pm$ 0.037}
        &  \makecell[l]{0.583 $\pm$ 0.002}
        &  \makecell[l]{\underline{2.488} $\pm$ 0.015}
        &  \makecell[l]{\textbf{5.827} $\pm$ 0.017}
        \\
        \textbf{ }
        &  \textbf{\makecell[l]{\fwname (Ours)}}
        &  \makecell[l]{\textbf{0.694} $\pm$ 0.002} 
        &  \makecell[l]{\textbf{0.754} $\pm$ 0.003} 
        &  \makecell[l]{\textbf{-9.462} $\pm$ 0.038}
        &  \makecell[l]{\textbf{0.794} $\pm$ 0.003}
        &  \makecell[l]{\textbf{2.077} $\pm$ 0.017}
        &  \makecell[l]{\underline{3.712}  $\pm$ 0.028}
        \\
        \bottomrule
        \\
    \end{tabular}}}
        \caption{
        {\textbf{Main results.} A comparison of seven baselines including Original, six baselines from REINVENT~4 \{MMP, Similarity $\geq 0.5$, Similarity $\in [0.5,0.7)$, Similarity $\geq 0.7$, Scaffold, Scaffold Generic\}, Molsearch, MIMOSA, DrugEx v3, 
        and \fwname on multiple objectives 
        based on 3CLPro and RTCB datasets with Tanimoto Similarity above 0.6. {The top two results are highlighted as \textbf{1st} and \underline{2nd}. Results are reported for 5 experimental runs.} 
        }
        }
        \label{exp:main_result}
\end{table*}
\vspace{+0.6cm}
\section{Theoretical Analysis}\label{sec:theory}
\fix{We now provide a theoretical analysis of \algname and prove its effectiveness and superiority over prior RL algorithm for both local and global optimizations.}
Recall that we have optimization target $J(\theta)$ defined in \eqref{eq:J(theta)} with advantage function $R^\AP$ defined in \eqref{eq:advantage_preference}.
Define $J_0(\pi)=\E_{X\sim \rho_0}[r_\AP^\pi(X)] = \E^\pi[R_c(Y_{1:T}) - R_c(X)]$ as the ``standard'' RL metric.
We say that BON \emph{strictly improves} over suboptimal molecule if
\begin{align}
    r_{\BON(j)}^{\AP}(Y_{1:T}; X) > r^{\AP}(Y_{1:T}, X), \quad \forall  j\in[T], 
    \label{eq:strict improvment}
\end{align}
for any $Y_{1:T}$ such that $R_c(Y_{1:T}) < \max_{Y_{1:T}'} R_c(Y_{1:T}')$.
\paragraph{\algname can Find the Optimizer.}
Our first result compare the maximizers of $J(\cdot)$ under the \algname framework to those of $J_0(\pi) = \E^\pi[R_c(Y_{1:T})]$.
\begin{lemma}
\label{lem:equivalence}
\vspace{+0.3cm}
    If BON finds a sequence that strictly improves over the current molecule in the sense of \eqref{eq:strict improvment},
    any policy $\pi^*$ maximizes $J(\pi)$ if and only if it maximizes the original reward $J_0(\pi)$.
\end{lemma}
Given the fact that these two optimization targets share the same optimizer, we next study the benefit of using our definition $J$ for gradient update.

\paragraph{Densifying the Reward Signal.}
We remark that using $J(\pi)$ has the advantage of densifying the reward signal, thus making policy optimization easier.
In fact, each $r_{\BON(j)}^\AP(Y_{1:T}, X)$ serves as a reward signal for choosing the next action $Y_{j+1}$ at state $(Y_{1:j}, X)$.
\begin{lemma}
\vspace{+0.2cm}
    \label{lem:gradient}
     Gradient $g$ defined in \eqref{eq:adv:preference} can be rewritten as
    \begin{align}
        g =& \frac {1} {2T} \sum_{t=1}^T \E_{X\sim\rho_0}^\pi \left[\nabla_\theta \log \pi_\theta(Y_{1:t}\given X)  r_{\BON(t)}^{\AP}(Y_{1:T}, X)\right] \notag\\
        & +  \frac 1 2 \cdot \E_{X\sim\rho_0}^\pi \left[\nabla_\theta \log \pi_\theta(Y_{1:T}\given X)  r^\AP(Y_{1:T}, X) \right]. \label{eq:grad_reform}
    \end{align}
\end{lemma}
The last term corresponds to the reward at the end of the generation of the molecule, while the first term provides ``partial'' reward on each generation step.
This gradient form suggests that incorporating partial molecule's advantage function densifies the reward signal along the trajectory, enabling the learning agent to explore more efficiently \citep{riedmiller2018learning, vecerik2017leveraging}. \fix{We defer the proof of \lemref{lem:equivalence} and \lemref{lem:gradient} to \appref{app:theory_proof}.}

\section{Experiments}\label{experiments}

{\paragraph{The language model.} 
We utilize the Byte Pair Encoding (BPE) method \citep{gage1994new_bpe, sennrich2015neural_bpe} to initially pre-train our tokenizer using raw SMILES strings and GPT-2-like Transformers for causal language modeling. We train on the standard 11M Drug-like Zinc dataset, excluding entries with empty scaffold SMILES. The dataset is divided into a 90/10 split for training and validation, respectively. (For more details, see \appref{app:pretrain_data}).}

{\paragraph{Baselines.} 
In this study, 
we employ several baseline models, including 
Molsearch \citep{sun2022molsearch}, a search-driven approach leveraging Monte Carlo Tree Search (MCTS) for molecular generation and optimization,  MIMOSA \citep{fu2021mimosa}, a graph-based method for molecular optimization based on sampling and \fix{DrugEx v3 \citep{liu2023drugex}, a scaffold-based drug optimization using transformer-based reinforcement learning}. 
\fix{Moreover, we integrate the state-of-the-art model, Mol2Mol~\citep{he2021molecular, he2022transformer} from REINVENT~4~\citep{loeffler2024reinvent},} 
which trains a transformer to adhere to the Matched Molecular Pair (MMP) guidelines~\citep{tyrchan2017matched}.
Specifically, given a set $\{\{X,Y,Z\}\}$, where $X$ represents the source molecule, $Y$ denotes the target molecule, and $Z$ signifies the property change between $X$ and $Y$, the model learns a mapping from $\{X, Z\} \in \ensuremath{\mathcal{X}} \times \ensuremath{\mathcal{Z}} \implies Y \in \ensuremath{\mathcal{Y}}$ during training.
REINVENT~4 defines six types of property changes for $Z$, including MMP for user-specified alterations, various similarity thresholds, and scaffold-based modifications where molecules share the same scaffold or a generic scaffold.}

\paragraph{Datasets.}
{Utilizing the latest Cancer and COVID dataset proposed in this paper \fix{(See Appendix \ref{app:dataset}) for RL fine-tuning across all baseline models and our proposed method}, which consists 1 million compounds from the ZINC15 dataset docked to the 3CLPro~(PDB ID: 7BQY) protein associated with SARS-CoV-2 and the RTCB (PDB ID: 4DWQ) human cancer protein.}
This newly proposed dataset is utilized for RL fine-tuning across all baseline models and are not employed in the pretraining phase. 
For pretraining, we rely on molecules from the ZINC database, filtering for Standard, In-Stock, and Drug-Like molecules, resulting in approximately 11 million molecules (See details in Appendix \ref{app:pretrain_data}). We formed molecular pairs from the ZINC dataset, adhering to the guidelines used for creating the pretraining corpora of each baseline method. The specific rules for generating our molecular pairs are outlined in Section \ref{sec:pretrain} of our approach.

\paragraph{Critics and evaluation metric.} 
\fix{In addition to the metrics introduced in section \ref{sec:4.2}, we further added the following two metrics:}
{
\textbf{Average Top 10\% Norm Reward:} It is the average of the normalized reward of the top 10\% of molecules.
}
\fix{\textbf{Average Norm Reward:} It is the average of the normalized values of all metrics across valid molecules. This is the most important metric.}

\subsection{Experimental results}

Table \ref{exp:main_result} demonstrates that the \fwname algorithm outperforms all competing baselines, including the original and six variants from the current leading method, REINVENT~4, across most performance metrics for both viral and cancer-related benchmarks. 
It improves diversified properties and
significantly enhancing the critical metric of average normalized reward. 
\fwname excels over REINVENT~4 primarily because REINVENT~4 concentrates on pretraining with restricted similarity and fails to effectively enhance the properties of generated drugs, limiting exploration of potentially high-reward molecular spaces. In contrast, \fwname employs \algname to explore high-reward spaces while maintaining reasonable similarity, increasing the probability of generating sequences with positive advantages and decreasing it for negative ones. Additionally, \algname optimizes both entire and partial molecules, facilitating both global and local optimizations, which results in quicker convergence and improved performance.

\begin{wrapfigure}[9]{r}{.2\textwidth}
    \vspace{-10pt}
    \centering
    \includegraphics[width=\linewidth]{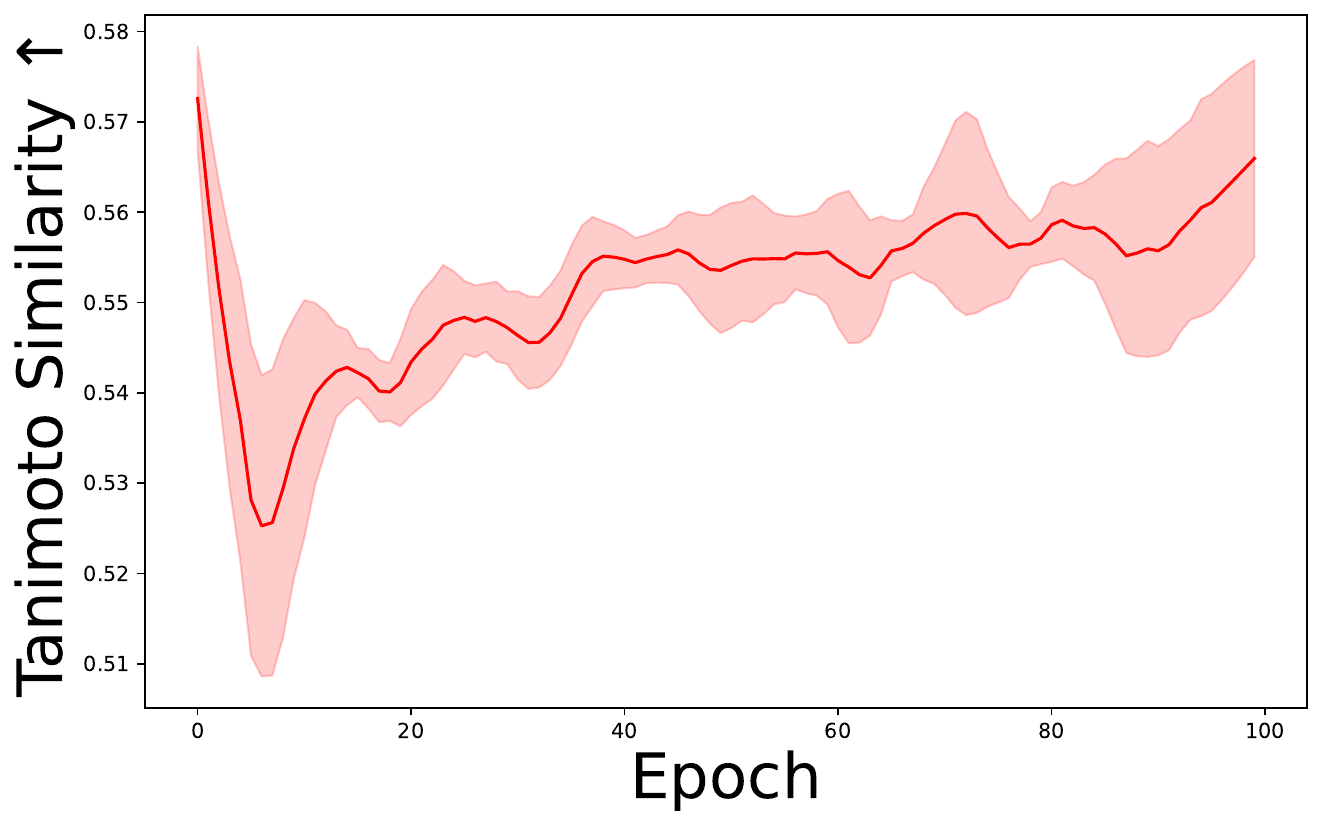}
    \vspace{-7mm}
    \caption{Tanimoto Similarity over five experimental runs.}\label{fig:exp:sim}
\end{wrapfigure}
\paragraph{\algname adjusting.} 
Note that the performance curve of Tanimoto similarity in~\figref{fig:exp:sim} initially decreases and then increases. This trend aligns ideally with the RL-based molecule generation improvement process.
The initial decrease occurs because \algname relaxes the structural constraints of the original molecule to achieve greater improvement in the diversified properties of the generated molecules. This causes the molecule to deviate from its original structure, leading to a decrease in Tanimoto similarity. 
Subsequently, there is a gradual increase in the trend as the generated molecules reach a decent level of diverse properties and begin optimizing their structure towards that of the original molecule, resulting in an increasing trend in Tanimoto similarity. Finally, the generated molecules not only improve the desired properties but also reduces the likelihood of drastic structural changes that might result in unsynthesizable compounds. 
{This process, illustrated in Figure \ref{fig:exp:sim}, showcases \algname's capability to automatically adjust the optimization of various properties to achieve global optimization.}

\begin{table*}[t!]
\setlength{\tabcolsep}{4pt}
   \centering
    { %
    \scalebox{0.62}{
    \begin{tabular}{l l l c c c c c c c c c c }
        \toprule
        \textbf{Target} %
        & \textbf{Algorithm}
        & \makecell[c]{\fix{Validity}~$\uparrow$}
        & {\makecell[c]{Avg \\ Norm Reward~$\uparrow$}}
        & {\makecell[c]{Avg Top 10 \% \\ Norm Reward~$\uparrow$}}
        & {\makecell[c]{Docking ~$\downarrow$}}
        & {\makecell[c]{Druglikeliness ~$\uparrow$}}
        & {\makecell[c]{Synthesizability ~$\downarrow$}}
        & {\makecell[c]{Solubility ~$\uparrow$}}
        \\
        \midrule
        \makecell[l]{\textbf{3CLPro}} %
        &  \textbf{\makecell[l]{\algname without partial}}
        &  \makecell[l]{\textbf{0.902}}        
        &  \makecell[l]{0.561}
        &  \makecell[l]{{0.666}}
        &  \makecell[l]{\textbf{-8.283}}
        &  \makecell[l]{0.614}
        &  \makecell[l]{2.740}
        &  \makecell[l]{3.597}
        \\
        (PDBID:7BQY)
        &  \textbf{\makecell[l]{\algname (Ours)}}
        &  \makecell[l]{{0.844}}
        &  \makecell[l]{\textbf{0.601}}
        &  \makecell[l]{\textbf{0.692}}
        &  \makecell[l]{-8.163}
        &  \makecell[l]{\textbf{0.676}}
        &  \makecell[l]{\textbf{2.381}}
        &  \makecell[l]{\textbf{3.673}}
        \\
        \bottomrule
        \textbf{RTCB}
        &  \textbf{\makecell[l]{\algname without partial}}
        &  \makecell[l]{{0.879}}
        &  \makecell[l]{0.592}
        &  \makecell[l]{{0.724}}
        &  \makecell[l]{-8.318}
        &  \makecell[l]{0.618}
        &  \makecell[l]{2.527}
        &  \makecell[l]{3.832}
        \\
        (PDBID:4DWQ)
        &  \textbf{\makecell[l]{\algname (Ours)}}
        &  \makecell[l]{\textbf{0.964}}        
        &  \makecell[l]{\textbf{0.694}}
        &  \makecell[l]{\textbf{0.754}}
        &  \makecell[l]{\textbf{-9.462}}
        &  \makecell[l]{\textbf{0.794}}
        &  \makecell[l]{\textbf{2.077}}
        &  \makecell[l]{\textbf{3.712}}
        \\
        \bottomrule
        \\
    \end{tabular}}}
        \caption{
        {\textbf{Ablation study.} 
        A comparison between \algname with and without the partial molecule improvement component shows that \algname with this component outperforms in most metrics.
        }}
        \label{exp:ablation}
\end{table*}

\paragraph{Ablation study on \algname.} 
\algname is distinguished from previous RL algorithms by its introduction of advantage preference with partial molecule improvement. In \tabref{exp:ablation}, we perform an ablation study on the partial molecule improvement component, showing that this component in \algname leads to performance enhancements across nearly all metrics, which aligned with our theoretical result of  densifying the reward signal. \fix{See \appref{spo_validity} for novelty and diversity ablation.}

\section{Conclusion}\label{sec:con}

We present the \fwname framework, which includes a LLM designed for drug optimization and \algname, a structured policy optimization algorithm—the novel RL finetuning algorithm tailored for drug optimization. 
We provide a rigorous theoretical analysis of \algname, demonstrating its effectiveness in aligning the LLM-based generator policy with desired objectives and performs efficient policy gradient updates based on the advantage preference. \algname seeks to achieve maximal improvement on desired properties based on the original drug while maintaining its necessary properties.
Moreover, we evaluate \fwname on SARS-CoV-2 and human cancer benchmarks, respectively. Our results reveal that our optimized compounds exhibit significant improvement over the original compounds and outperform the current state of the art in multiple properties.
Our research opens up new possibilities for enhancing drug optimization and inspires future investigations into addressing challenges within the realm of drug optimization. This includes exploring areas like the integration of graph information.
 We leave this extension to future work.

\subsubsection*{Acknowledgements}
We thank A. Vasan, A. Brace, O. Gokdemir, T. Brettin and F. Xia for initial discussion. This work is supported  by the RadBio-AI project (DE-AC02-06CH11357), U.S. Department of Energy Office of Science, Office of Biological and Environment Research; Improve project under contract (75N91019F00134, 75N91019D00024, 89233218CNA000001, DE-AC02-06-CH11357, DE-AC52-07NA27344, DE-AC05-00OR22725); Exascale Computing Project (17-SC-20-SC), a collaborative effort of the U.S. Department of Energy Office of Science and the National Nuclear Security Administration; and the National Science Foundation under Grant No. IIS 2313131, IIS 2332475, and DMS 2413243.

\bibliographystyle{plainnat}
\bibliography{reference}

\appendix

\onecolumn

\section{Appendix}

{
\subsection{Pre-training and fine-tuning dataset}\label{app:pretrain_data} 

We utilized the ZINC dataset, filtering for Standard, In-Stock, and Drug-Like molecules, resulting in approximately 11 million molecules.
The new pre-training dataset is constructed by randomly selecting two molecules from the ZINC dataset that meet the proposed criteria (as described in Equation \ref{eq:corpus}). This new pre-training dataset comprises 10 million molecules, with a 90/10 training/validation split.

\fix{For fine-tuning, we employ one million compounds from the ZINC15 dataset, docked to the 3CLPro protein (PDB ID: 7BQY), which is linked to SARS-CoV-2, and the RTCB protein (PDB ID: 4DWQ), associated with human cancer. These data are obtained from the latest Cancer and COVID dataset by \citet{liu2023drugimprover} and are used across all baselines.
}
}

{
\subsection{Generation with finetuned model} \label{app:generation} 
The epoch with highest historical average normalized reward (as detailed in Section \ref{experiments}) is selected for generation. 
With this epoch and corresponding weights, we apply TOPPK\citep{liu2024erp} for generation. 
}

{
\subsection{Baseline REINVENT 4} \label{app:reinvent}
Following are detailed description of six different kinds of property change $Z$ included in REINVENT \citet{he2022transformer, he2021molecular}
\begin{itemize}
    \item {MMP:} There are user-defined desirable property changes between molecules $X$ and $Y$.
    \item {Similarity $\geq 0.5$:} The Tanimoto similarity between molecules $X$ and $Y$ exceeds 0.5.
    \item {Similarity $\in [0.5, 0.7)$:} The Tanimoto similarity for the pair $\left(X, Y\right)$ lies between 0.5 and 0.7.
    \item {Similarity $\geq 0.7$:} The Tanimoto similarity between molecules $X$ and $Y$ exceeds 0.7.
    \item {Scaffold:} Molecules $X$ and $Y$ share the same scaffold.
    \item {Scaffold generic:} Molecules $X$ and $Y$ share the same generic scaffold.
\end{itemize}}

\subsection{Binding sites of 3clpro and RTCB}
\begin{figure*}[ht!]

    \begin{subfigure}{.5\textwidth}
        \centering
        \includegraphics[%
        width=6cm,  clip={0,0,0,0}]{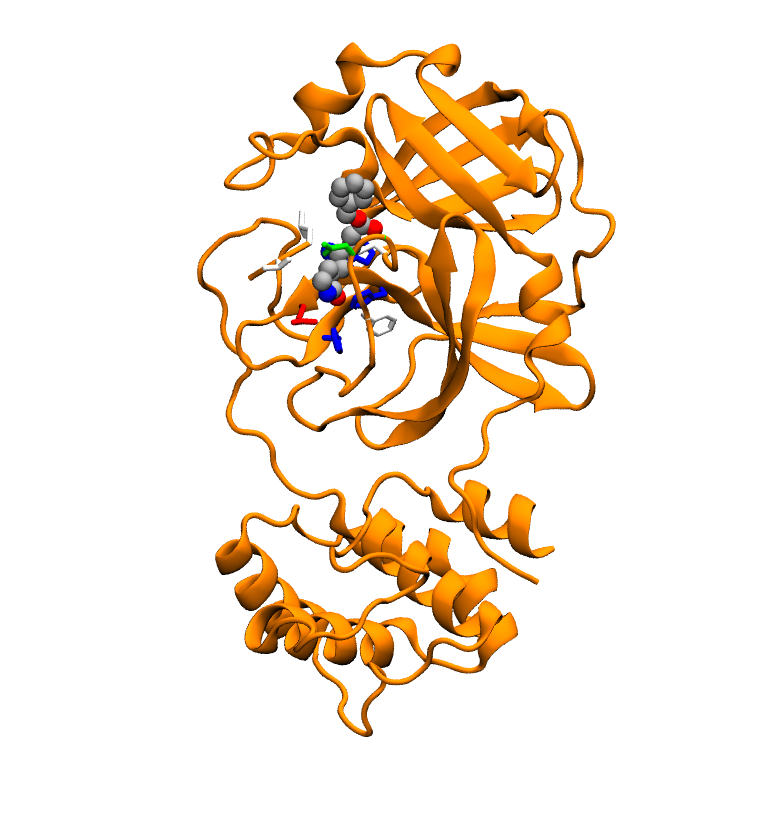}
        \caption{3CLPro.}\label{fig:exp:combine_expertise}
    \end{subfigure}\hfil
    \begin{subfigure}{.5\textwidth}
        \centering
        \includegraphics[%
        width=6cm,  clip={0,0,0,0}]{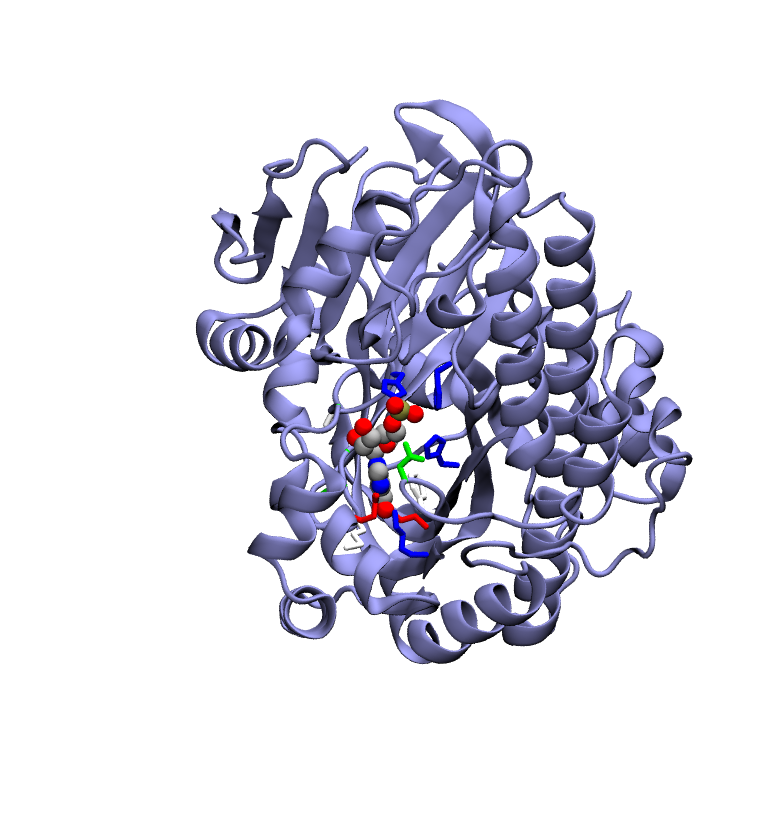}
        \caption{RTCB.}\label{fig:exp:sample}
    \end{subfigure}\hfil
    \caption{The binding sites of proteins 3CLPro (PDB ID: 7BQY)~(\textbf{Left}) and RTCB (PDB ID: 4DWQ)~(\textbf{Right}). {Open Eye software are used to identify atoms around the crystallized compound as binding sites.}     
    }
\end{figure*}

\subsection{Surrogate model}\label{app:surrogate_model}
{The surrogate model~\citep{vasan23} is a simplified variant of a BERT-like transformer architecture, commonly utilized in natural language processing tasks. Within this model, tokenized SMILES strings are initially inputted and subsequently undergo positional embedding. The outputs are then fed into a series of five transformer blocks, each comprising a multi-head attention layer (with 21 heads), a dropout layer, layer normalization with residual connection, and a feedforward network. This feedforward network is composed of two dense layers followed by dropout and layer normalization with residual connection. Following the stack of transformer blocks, a final feedforward network is employed to generate the predicted docking score.
\fix{ The validation $r^2$ values are 0.842 for 3CLPro and 0.73 for the RTCB dataset.}

}

\subsection{Computing infrastructure {and wall-time comparison}}\label{app:computing_infrastructure}

{We trained our docking surrogate models using 4 nodes of a supercomputer, where each node contains 64 CPU cores and 4 A100 GPUs. The training time for each model was approximately 3 hours.
We conducted other experiments on a cluster that includes CPU nodes with approximately 280 cores and GPU nodes with approximately 110 Nvidia GPUs, ranging from Titan X to A6000, mostly set up in 4- and 8-GPU configurations. 
Pretraining utilizes 8 GPUs, while \algname uses a single GPU. Both processes employ either V100 or A100 GPUs. Based on the computing infrastructure, we obtained the wall-time comparison in \tabref{table:wall-time} as follows.
}

 \begin{table*}[ht!]
 {
    \centering
    {\scriptsize
    \scalebox{1}{
    \begin{tabular}{l c c  }
        \toprule
        \textbf{Methods}
        & {\makecell[c]{Total Run Time}}
        \\
        \midrule
        \textbf{\makecell[l]{Pretrain}}
        &  \makecell[r]{24h}
        \\
        \textbf{\makecell[l]{\algname}}
        &  \makecell[r]{8h}
        \\
        \bottomrule
    \end{tabular}}}
    \caption{{Wall-time comparison between different methods.} }
        \label{table:wall-time}   
        }
\end{table*}

\subsection{Hyperparameters and architectures}\label{app:hyperparameters}
Table \ref{app:tab:hyperparams} provides a list of hyperparameter settings we used for our experiments.
A selection of 1280 molecules from each of the RTCB and 3CLPro datasets, with docking scores ranging from -14 to -6, is used for \algname finetuning and experimentation. This range is based on \citep{liu2024erp}.
{Furthermore, when calculating the average normalized reward for the original molecule, where similarity is not considered, we assign a weight of $[0.25] \times 4$ to docking, druglikeliness, synthesizability, and solubility.}
{Moreover, when the generated SMILES is invalid, meaning that calculating the reward $R_c$ is not possible, we have two options: the first is to directly subtract the reward of the original SMILES (i.e., $-R_c(X)$), or alternatively, we can consider the advantage preference as zero.}

\begin{table*}[h!]
    {
    \centering
    {\scriptsize
    \scalebox{1}{
    \begin{tabular}{c c }
        \toprule
        \textbf{Parameter} &  \textbf{Value} 
        \\
        \midrule
        {\makecell[l]{Pretraining}}
        \\
        \midrule
        {\makecell[l]{\quad Learning rate}} &  \makecell[c]{$5 \times e^{-5}$}
        \\
        \midrule
        {\makecell[l]{\quad Batch size}} &  \makecell[c]{$24$}
        \\
        \midrule
        {\makecell[l]{\quad Optimizer}} &  \makecell[c]{Adam}
        \\
        \midrule
        {\makecell[l]{\quad \# of Epochs}} &  \makecell[c]{$10$} \\
        \midrule
        {\makecell[l]{\quad Model \# of Params}} &  \makecell[c]{$124M$} \\

        \bottomrule
    \end{tabular}}}
        \caption{{{Hyperparameters for pretraining}}. }
        \label{app:tab:hyperparams}
        }
\end{table*}

\begin{table*}[h!]
    {
    \centering
    {\scriptsize
    \scalebox{1}{
    \begin{tabular}{c c }
        \toprule
        \textbf{Parameter} &  \textbf{Value} 
        \\
        \midrule
        {\makecell[l]{Shared}}
        \\
        \midrule
        {\makecell[l]{\quad \# of Molecules Optimized}} &  \makecell[c]{$256$}
        \\
        \midrule
        {\makecell[l]{\quad Learning Rate}} &  \makecell[c]{$1 \times 10^{-5}$}
        \\
        \midrule
        {\makecell[l]{\quad Optimizer}} &  \makecell[c]{Adam}
        \\
        \midrule
        {\makecell[l]{\quad \# of Epochs for Training}} &  \makecell[c]{$100$}
        \\
        \midrule
        {\makecell[l]{\quad Batch size}} &  \makecell[c]{$64$}
        \\
        \midrule
        {\makecell[l]{\quad Best-of-N}} &  \makecell[c]{$[4,6,8]$}
        \\
        \midrule
        {\makecell[l]{\quad TopK}} &  \makecell[c]{$[10,15,20]$}
        \\
        \midrule
        {\makecell[l]{\quad TopP}} &  \makecell[c]{$[0.85,0.9,0.95]$}
        \\
        \midrule
        {\makecell[l]{\algname Objective Weight}}
        \\
        \midrule
        {\makecell[l]{\quad Tamimoto Similarity}} &  \makecell[c]{$[0.2,0.4,0.6,0.8]$}
        \\
        \midrule
        {\makecell[l]{\quad Other Four Objectives}} &  \makecell[c]{$(1-W(Sim)) / 4$}
        \\
        \midrule
        {\makecell[l]{\algname Other}}
        \\
        \midrule
        {\makecell[l]{\quad Fingerprint Size}} &  \makecell[c]{$1024$}
        \\
        \midrule
        {\makecell[l]{\quad Normalize Min/Max}} &  \makecell[c]{$[-10, 10]$}
        \\
        \midrule
        {\makecell[l]{Advantage preference with \\ invalid generated SMILES}}
        \\
        \midrule
        {\makecell[l]{\quad 3CLPro}} &  \makecell[c]{$[0,-R_c(X)]$}
        \\
        \midrule
        {\makecell[l]{\quad RTCB}} &  \makecell[c]{$[0,-R_c(X)]$}
        \\

        \bottomrule
    \end{tabular}}}
        \caption{{{Hyperparameters for \algname}}. }
        \label{app:tab:hyperparams}
        }
\end{table*}

{
\subsection{Ablation study of novelty and diversity}\label{spo_validity}

We further explore the novelty and diversity of SPO, both with and without partial molecule enhancements. Novelty is assessed by verifying if the generated molecule/SMILES is present in the original dataset, assigning a value of 0 if it exists and 1 if it does not. Diversity is measured by examining if the generated molecule/SMILES is repeated within the generated dataset, indicating a duplication if two distinct prompts or molecules yield the same outcome. The findings demonstrate that the molecules created through our method are completely new in comparison to the original molecules. Moreover, our technique has attained a decent level of diversity.

\begin{table}[h!]
\centering
\begin{tabular}{|l|l|c|c|}
\hline
\textbf{Data} & \textbf{Method}        & \textbf{Novelty} & \textbf{Diversity} \\ \hline
3CLPro        & SPO w/o partial        & 1               & 0.98               \\ \cline{2-4} 
              & SPO                    & 1               & 0.95               \\ \hline
RTCB          & SPO w/o partial        & 1               & 0.98               \\ \cline{2-4} 
              & SPO                    & 1               & 0.69               \\ \hline
\end{tabular}
\caption{Comparison of SPO with and without partial molecule improvements across cancer and covid datasets in terms of novelty and diversity.}
\label{tab:validity_comparison}
\end{table}
}

\fix{
\subsection{Training corpus visualization}\label{corpus_viz}
For better understanding the training corpus, \figref{fig:corpus_fiz} shows an example and corresponding visualization towards training corpus described in \eqref{eq:corpus}. The two selected molecules/SMILES have the similarity of 0.52.
}
\begin{figure}[ht!]
    \centering
    \includegraphics[width=0.5\linewidth]{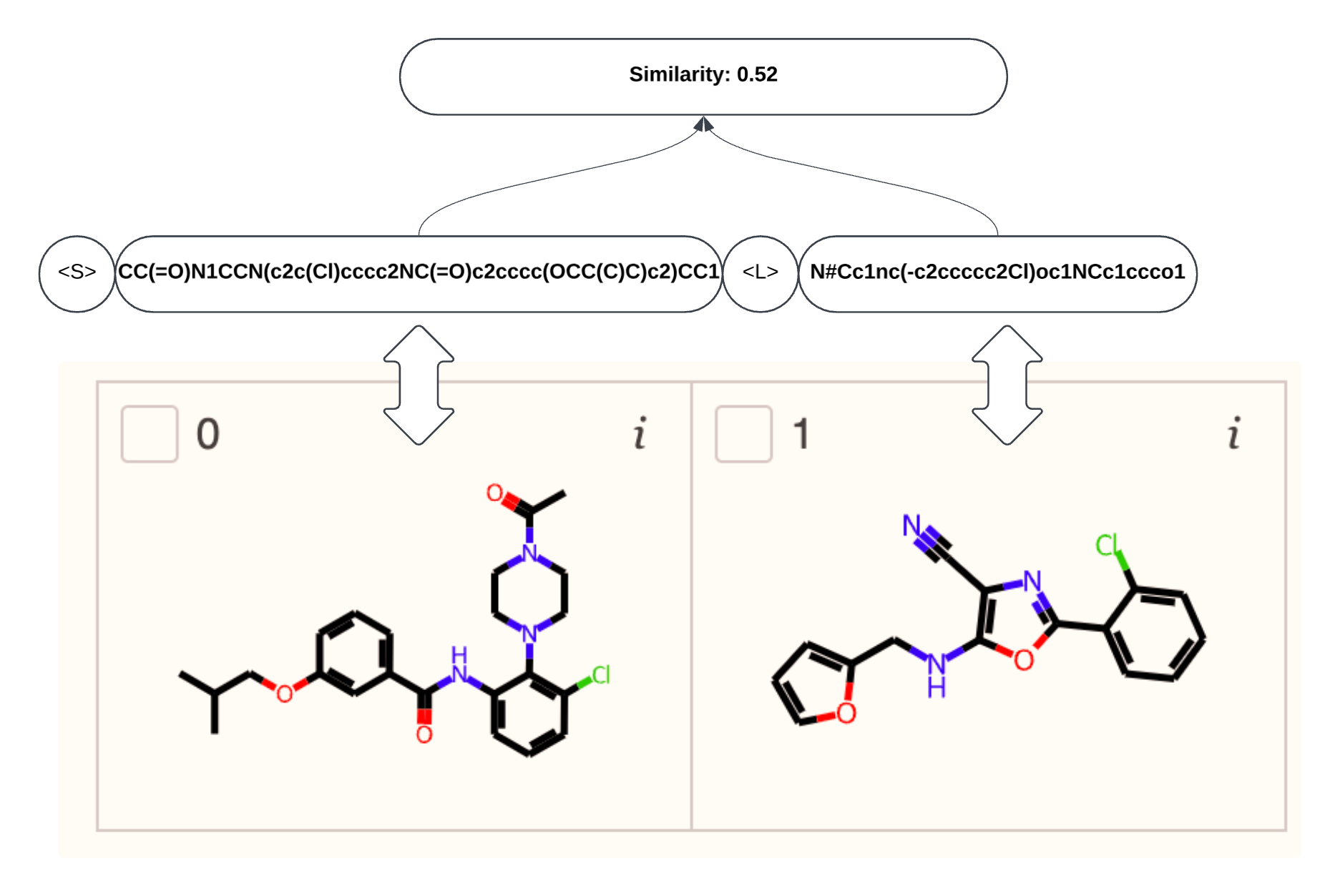}
    \caption{\fix{Training corpus example and visualization}}
    \label{fig:corpus_fiz}
\end{figure}

\fix{
\subsection{Baselines with RL fine-tuning}

The initial version of Reinvent4 \XL{@songhao check version number}\citep{he2021molecular, he2022transformer} only introduced pre-trained models, and the later updated version of Reinvent4 \citep{loeffler2024reinvent}, which includes Mol2Mol~\citep{he2021molecular, he2022transformer} as one of four models, stated that Reinvent could perform RL fine-tuning through REINFORCE~\citep{williams1992simple} algorithm without providing empirical results. In this work, we conducted experiments for both Reinvent with pre-training only and Reinvent with RL fine-tuning. For Mol2Mol~\citep{he2021molecular, he2022transformer} with pre-training only, we followed different pre-trained rules outlined in their paper to pretrained the models and used them as various baseline models; meanwhile, we used the same pre-trained ZINC dataset as in our approach. Molsearch and Mimosa, on the other hand, focuses more on optimizting sampling process. Molsearch uses Monte Carlo tree search (MCTS) to optimize molecular properties; MIMOSA \citep{fu2021mimosa} designed a Markov Chain Monte Carlo (MCMC) based molecule sampling method that enables efficient sampling from a target distribution. The results are provided in  \tabref{exp:main_result}; DrugEx v3 \citep{liu2023drugex} employs graph transformers with scaffold constraints to refine molecular structures, leveraging reinforcement learning to enhance the desired molecular properties. In addition, we also conducted experiments for Reinvent 4~\citep{loeffler2024reinvent} with additional RL fine-tuning, using the same offline dataset we proposed in the paper, the same scoring function, and the same number of training epochs as our approach. In the cancer dataset, our proposed method outperforms all variants of Reinvent 4~\citep{loeffler2024reinvent} with RL finetuning. And in the COVID-19 dataset, our proposed method still outperforms almost all variants. Therefore, our method surpasses the performance of both the pre-trained-only Mol2Mol~\citep{he2021molecular, he2022transformer} and the version~\citep{loeffler2024reinvent} that underwent REINFORCE fine-tuning.

\fwname also surpasses REINVENT4 with RL fine-tuning. This is because REINVENT4 employs the REINFORCE, a conventional RL approach, which does not account for improvements over the original molecule. In contrast, our proposed SPO algorithm is specifically designed for drug optimization toward original given molecule. It incorporates the concept of advantage preference and partial molecule components to optimize target molecules more effectively.

}

\fix{
\section{Dataset details}\label{app:dataset}
\fix{For each dataset proposed, it is a orderable subset of the ZINC15 dataset. Creating these subsets was mainly a manual process, involving the identification of compounds that are either in stock or can be shipped within three weeks from various suppliers. Subsequently, we performed a random sampling to select 1 million compounds.

For proteins with available structures containing bound ligands, we used X-ray crystallographic data to locate ligand density regions and defined the pocket as a rectangular box enclosing that area. For proteins without bound ligands, we employed FPocket to identify the top-ranked pocket and similarly defined the pocket with a rectangular box around that region. And therefore for each dataset we proposed, only one pocket is used for docking. 
The validation $r^2$ values are 0.842 for 3CLPro and 0.73 for the RTCB dataset (two datasets used in section \ref{experiments}). 

}

The datasets created in this work including the following files:

\begin{itemize}
    \item {ST$\_$MODEL:}  The trained surrogate model for SARS-CoV-2 proteins.
    \item {ST$\_$MODEL$\_$rtcb:} The trained surrogate model for RTCB Human-Ligase cancer target.
    
    \item {24 *.csv files for SARS-CoV-2 proteins under folder data/COVIDRec:} The training and validation SMILES string data docked on SARS-CoV-2 receptor including 3CLPro$\_$7BQY$\_$A$\_$1$\_$F, NPRBD$\_$6VYO$\_$AB$\_$1$\_$F, NPRBD$\_$6VYO$\_$A$\_$1$\_$F, NPRBD$\_$6VYO$\_$BC$\_$1$\_$F, NPRBD$\_$6VYO$\_$CD$\_$1$\_$F, NPRBD$\_$6VYO$\_$DA$\_$1$\_$F, NSP10-16$\_$6W61$\_$AB$\_$1$\_$F, NSP10-16$\_$6W61$\_$AB$\_$2$\_$F, NSP10$\_$6W61$\_$B$\_$1$\_$F, NSP15$\_$6VWW$\_$AB$\_$1$\_$F, NSP15$\_$6VWW$\_$A$\_$1$\_$F, NSP15$\_$6VWW$\_$A$\_$2$\_$F, NSP15$\_$6W01$\_$AB$\_$1$\_$F, NSP15$\_$6W01$\_$A$\_$1$\_$F, NSP15$\_$6W01$\_$A$\_$2$\_$F, NSP15$\_$6W01$\_$A$\_$3$\_$H, NSP16$\_$6W61$\_$A$\_$1$\_$H, Nsp13.helicase$\_$m1$\_$pocket2, Nsp13.helicase$\_$m3$\_$pocket2, PLPro$\_$6W9C$\_$A$\_$2$\_$F, RDRP$\_$6M71$\_$A$\_$2$\_$F, RDRP$\_$6M71$\_$A$\_$3$\_$F, RDRP$\_$6M71$\_$A$\_$4$\_$F, RDRP$\_$7BV1$\_$A$\_$1$\_$F.
    
    \item {5 *.csv files for human cancer proteins under folder data/CancerRep:} The training and validation SMILES string data docked on human cancer proteins including 6T2W, NSUN2, RTCB, WHSC, WRN.
    \item {Each folder in data/COVIDRec and data/CancerRep includes:
model.weights.h5: model weights
SMILES*.csv: 1 million SMILES their docking scores.
We also provide code within data/SurrogateInf on how to use the surrogate model for inference.}

    \item {3CLPro$\_$7BQY$\_$A$\_$1$\_$F.oeb:} The 3CLPro OpenEye receptor file.
    \item {rtcb-7p3b-receptor-5GP-A-DU-601.oedu:}  The RTCB OpenEye receptor file.
    \item {  We include an extended dataset of 1 million SMILES strings from the ZINC15 dataset, their docking scores (as determined by OpenEye FRED) to 24 COVID and 5 cancer-target receptors and surrogate model weights for each corresponding receptor.}
    \item We provide code within data/SurrogateInf on how to use the surrogate model for inference.
\end{itemize}}

\section{Proofs of the theoretical results}\label{app:theory_proof}
\begin{proof}[Proof of Lemma \ref{lem:equivalence}]\label{app:lem:equivalence}
    For simplicity, let us take out the shift terms $R_c(X)$ in $r^\AP$ and $R_c(X_{1:T})$ in $r^\AP_{\BON(j)}$ for a while when defining $J$. 
    Since the shift term $R_c(X)$ (or its BON counterpart) are  independent of the current policy $\pi$, such an operation does not influence the definition of optimal policy for $J$.
    One can always split $J=\frac 1 2 J_{\BON} + \frac 1 2 J_0$, where 
    $J_{\BON} = \E_\pi \E_{j\in \mathcal{U}([T])} r_{\BON(j)}^{\AP}(Y_{1:T}, X)$. 
    Take $\pi^\star$ to be an optimizer of $J_0$. By definition, 
    \begin{align*}
        J_0(\pi^\star) \ge J(\pi) \ge J_0(\pi)
    \end{align*}
    as BON cannot be worse than the current molecule. 
    Thus, any policy that maximizes $J_0$ should also be a maximizer to $J$. 
    On the other hand, notice that 
    \begin{align}
        J_0(\pi^\star) - J(\pi) \ge \frac 1 2 \cdot (J_0(\pi^\star) - J_0(\pi)) \ge 0
    \end{align}
    since the BON reward should not exceed the optimal reward.
    Hence, any policy that maximize $J$ should also maximize $J_0$ since the optimizer of $J$ gives optimal value equal to $J_0(\pi^\star)$. 
    Hence, we prove the claim.
\end{proof}

\begin{proof}[Proof of Lemma \ref{lem:gradient}]\label{app:lem:gradient}
Denote by $\pi_\theta$ the policy parameterized by $\theta$.
When doing policy optimization, we note that 
\begin{align}
    g &= \E_{X\sim\rho_0}^\pi  \left[\nabla_\theta \log \pi_\theta(Y_{1:T}\given X) R^\AP (Y_{1:T}, X) \right] \notag\\
   &= \E_{X\sim\rho_0}^\pi \left[\nabla_\theta \log \pi_\theta(Y_{1:T}\given X) \left(\frac 1 2 r^\AP(Y_{1:T}, X) + \frac 1 2 \E_{j\in \mathcal{U}([T])} r_{\BON(j)}^{\AP}(Y_{1:T}, X)\right)\right] \notag\\
   &= \frac {1} {2T} \sum_{t=1}^T \E_{X\sim\rho_0}^\pi \left[\left(\nabla_\theta \log \pi_\theta(Y_{1:t}\given X) + \nabla_\theta \log \pi_\theta(Y_{t+1:T}\given Y_{1:t}, X) \right) r_{\BON(t)}^{\AP}(Y_{1:T}, X)\right] \notag\\
   &\qquad +  \frac 1 2 \cdot \E_{X\sim\rho_0}^\pi \left[\nabla_\theta \log \pi_\theta(Y_{1:T}\given X)  r^\AP(Y_{1:T}, X) \right] \notag\\
   &= \frac {1} {2T} \sum_{t=1}^T \E_{X\sim\rho_0}^\pi \left[\nabla_\theta \log \pi_\theta(Y_{1:t}\given X)  r_{\BON(t)}^{\AP}(Y_{1:T}, X)\right] \notag\\
   &\qquad +  \frac 1 2 \cdot \E_{X\sim\rho_0}^\pi \left[\nabla_\theta \log \pi_\theta(Y_{1:T}\given X)  r^\AP(Y_{1:T}, X) \right], 
\end{align}
where the last equality holds by noting that 
$$\E^\pi[\nabla_\theta \log \pi_\theta(Y_{t+1:T}\given Y_{1:t}, X) \given Y_{1:t}, X]=0.$$
\end{proof}

\clearpage

\end{document}